\documentclass[10pt,twocolumn,letterpaper]{article}

\usepackage[pagenumbers]{cvpr} 

\usepackage{graphicx}
\usepackage{amsmath}
\usepackage{amssymb}
\usepackage{amsthm}
\usepackage{booktabs}
\usepackage[linesnumbered,ruled,vlined]{algorithm2e}
\usepackage{algorithmic}
 \usepackage{multirow}
\usepackage{makecell} 
\usepackage{siunitx}

\newcommand{\p}{\mathbb{P}}
\newcommand{\e}{\mathbb{E}}

\usepackage[pagebackref,breaklinks,colorlinks]{hyperref}

\newtheorem{theorem}{Theorem}

\usepackage[capitalize]{cleveref}
\crefname{section}{Sec.}{Secs.}
\Crefname{section}{Section}{Sections}
\Crefname{table}{Table}{Tables}
\crefname{table}{Tab.}{Tabs.}
\newcommand{\AlgName}{\textsc{brep-MI}\xspace}

\usepackage{xcolor}

\newcommand\tablescale{0.9}
\def\cvprPaperID{11210} 
\def\confName{CVPR}
\def\confYear{2022}

\begin{document}

\title{Label-Only Model Inversion Attacks via Boundary Repulsion}

\author{
Mostafa Kahla\\
Virginia Tech\\
{\tt\small kahla@vt.edu}
\and
Si Chen\\
Virginia Tech\\
{\tt\small chensi@vt.edu}
\and
Hoang Anh Just\\
Virginia Tech\\
{\tt\small just@vt.edu}
\and
Ruoxi Jia\\
Virginia Tech\\
{\tt\small ruoxijia@vt.edu}
}
\maketitle


\begin{abstract}

Recent studies show that the state-of-the-art deep neural networks are vulnerable to model inversion attacks, in which access to a model is abused to reconstruct private training data of any given target class. Existing attacks rely on having access to either the complete target model (whitebox) or the model's soft-labels (blackbox). 
However, no prior work has been done in the harder but more practical scenario, in which the attacker only has access to the model's predicted label, without a confidence measure. 
In this paper, we introduce an algorithm, Boundary-Repelling Model Inversion (\AlgName), to invert private training data using only the target model's predicted labels. The key idea of our algorithm is to evaluate the model's predicted labels over a sphere and then estimate the direction to reach the target class's
 centroid. Using the example of face recognition, we show that the images reconstructed by \AlgName successfully reproduce the semantics of the private training data for various datasets and target model architectures. We compare \AlgName with the state-of-the-art whitebox and blackbox model inversion attacks and the results show that despite assuming less knowledge about the target model, \AlgName outperforms the blackbox attack and achieves comparable results to the whitebox attack.

\end{abstract}
\vspace{-0.8em}

\section{Introduction}
Machine learning (ML) algorithms are often trained on private or sensitive data, such as face images, medical records, and financial information. Unfortunately, since ML models tend to memorize information about training data, even when stored and processed securely, privacy information can still be exposed through the access to the models~\cite{rigaki2020survey}. Indeed, the prior study of privacy attacks has demonstrated the possibility of exposing training data at different granularities, ranging from ``coarse-grained" information such as determining whether a certain point participate in training~\cite{shokri2017membership,long2018understanding, nasr2019comprehensive, hayes2019logan} or whether a training dataset satisfies certain properties~\cite{ganju2018property, melis2019exploiting}, to more ``fine-grained'' information such as reconstructing the raw data~\cite{fredrikson2015model,aivodji2019gamin,zhang2020secret,chen2021knowledgeenriched}. 

In this paper, we focus on model inversion (MI) attacks, which goal is to recreate training data or sensitive attributes given the access to the trained model. MI attacks cause tremendous harm due to the ``fine-grained" information revealed by the attacks. For instance, MI attacks applied to personalized medicine prediction models result in the leakage of individuals' genomic attributes~\cite{fredrikson2014privacy}. Recent works show that MI attacks could even successfully reconstruct high-dimensional data, such as images. For instance, \cite{zhang2020secret,chen2021knowledgeenriched,fredrikson2015model,yang2019adversarial} demonstrated the possibility of recovering an image of a person from a face recognition model given just their name.

Existing MI attacks have either assumed that the attacker has the complete knowledge of the target model or assumed that the attack can query the model and receive model's output as confident scores. The former and the latter are often referred to as the whitebox and the blackbox threat model, respectively. The idea underlying existing whitebox MI attacks~\cite{zhang2020secret,chen2021knowledgeenriched} is to synthesize the sensitive feature that achieves the maximum likelihood under the target model. The synthesis is implemented as a gradient ascent algorithm. By contrast, existing blackbox attacks~\cite{aivodji2019gamin, razzhigaev2020black} are based on training an attack network that predicts the sensitive feature from the input confidence scores. Despite the exclusive focus on these two threat models, in practice, ML models are often packed into a blackbox that only produces hard-labels when being queried. This label-only threat model is more realistic as ML models deployed in user-facing services need not expose raw confidence scores. However, the design of label-only MI attacks is much more challenging than the whitebox or blackbox attacks given the limited information accessible to the attacker.

\begin{figure}[!htbp]
  \centering
  \includegraphics[width=0.99\linewidth]{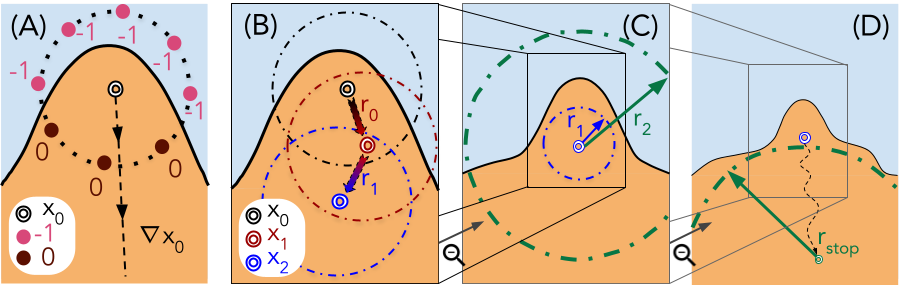}
  \caption{Intuitive explanation of \AlgName. (A) Query the labels over a sphere and estimate the direction on the sphere that can potentially lead to the target label class. (B) Update the synthesized image according to the estimated direction. Alternate between the estimation and update until the sphere fits into the target class. (C) Increase the radius of the sphere. (D) Repeat the steps above until the attack hits some query budget.
}
  \label{fig:example_attacks}
\end{figure}

In this paper, we introduce, \AlgName, a general algorithm for MI attack in the label-only setting, where the attacker can make queries to the target model and obtain hard labels, instead of confidence scores. Similar to the main idea of whitebox attacks, we still try to synthesize the most likelihood input for the target class under the target model. However, in the label-only setting, we cannot directly calculate the gradient information and leverage it to guide the data synthesis. Our key insight to resolve this challenge is that a high-likelihood region for a given class often lies at the center of the class and is far away from any decision boundaries. Hence, we design an algorithm that allows the synthesized image to iteratively move away from the decision boundary, as illustrated in Figure~\ref{fig:example_attacks}. Specifically, we first query the labels over a sphere and estimate the direction on the sphere that can potentially lead to the target label class (A). We progressively move according to estimated directions until the sphere fits into the target class (B). We then increase the radius of the sphere (C) and repeat the steps above until the attack hits some query budget (D). We theoretically prove that for linear target models, the direction estimated from hard labels queried on the spheres aligns with the gradient direction. We empirically show that \AlgName can also lead to successful attacks against deep neural network-based target models. In particular, the efficacy of the attack is even higher than the existing blackbox attacks and comparable to the existing whitebox attacks.


Our contributions can be summarized as follows: \textbf{(1)} We propose the first algorithm for label-only model inversion attacks. \textbf{(2)} We provide theoretical justification for the algorithm in the linear target model case by proving the updates used in our algorithm align with the gradient and also analyze the error of alignment for nonlinear models. \textbf{(3)} We evaluate the attack on a range of model architectures, datasets, and show that despite exploiting less information about the target model, our attack outperforms the confidence-based blackbox attack by a large margin and achieves comparable performance to the state-of-the-art whitebox attack. 
Besides, we will release the data, code, and models to facilitate future research.



\section{Related Work}

\paragraph{Model Inversion Attacks.} Model inversion attempts to reconstruct from partial up to full training sample. Typically, MI attacks can be formalized as an optimization problem, which goal is to find the sensitive feature value that achieves the highest likelihood under the model been attacked. However, when the target model is a deep neural network (DNN) or the private data lie in high-dimensional space, such optimization problem becomes non-convex and directly solving it via gradient descent may result in poor attack performance \cite{fredrikson2015model}; for example, when attacking a face recognition model, the recovered images are blurry and do not contain much private information. Recent work \cite{zhang2020secret} proposes a GAN-based MI attack method which is effective on DNNs. In particular, they learn a generic prior from public data via GAN and solve the optimization problem over the latent space rather than the unconstrained ambient space. However, their attack method does not fully exploit private information contained in the target model at the stage of training GAN. \cite{chen2021knowledgeenriched} significantly improves the attack performance through a special design of GAN which can distill knowledge from the target model; as a result, the generated images align better with the private distribution. They further improve the performance by ensuring that both the recovered image and its neighboring images have high likelihood. 
While \cite{zhang2020secret,chen2021knowledgeenriched} achieve success on attacking various models and datasets, their attacks rely on whitebox access to the model. In many cases, the attacker can only make prediction queries against a model, but not actually download the model, which motivates the study of blackbox MI attacks.
\cite{yang2019adversarial} analyzes the blackbox setting and proposes an attack model which swaps the input and prediction vector of the target model to perform model inversion. \cite{aivodji2019gamin} proposes to train a GAN and a surrogate model simultaneously, with the GAN generating inputs that resemble private training data and the surrogate model mimicking the target model's behavior. All of the blackbox attacks make an assumption that prediction confidences of the target model are revealed to the attacker. However, it is more practical in real-world setting that an adversary, who only makes queries to the model, can only obtain the hard labels, without confidence scores.
From this aspect, we aim to provide an effective MI attack method that only requires access to the hard label, which we refer to as label-only MI attacks.

%
%
%
%

\vspace{-0.9em}

\paragraph{Other Privacy Attacks.} 
Asides from MI, there are two other categories of privacy attacks that allow adversaries to gain unauthorized information from the target model and its data. 
In a membership inference attack, the attacker attempts to evaluate whether a certain point is used in the target model's training. This attack technique was introduced by \cite{shokri2017membership} who created multiple shadow models to estimate the target model. \cite{long2018understanding, nasr2019comprehensive, hayes2019logan} pointed out that the membership inference attack exploits the overfitting of specific data points. Interestingly, \cite{choquette2021label} performs a membership inference attack under same setting as our \AlgName attack and notes that the viable defense against such an attack is via differential privacy (DP). DP~\cite{dwork2014algorithmic, abadi2016deep} ensures that the trained model is stable to the change of any single record in the training set. 
However, with differential privacy, the target model's test accuracy will significantly degrade. Additionally, property inference attacks aim to infer from the properties about the training dataset~\cite{ganju2018property}. Compared to these attacks, MI is arguably more challenging as the information it attempts to recover is higher in resolution.

\section{Threat Model}

\paragraph{Attack goal.} In MI attacks, given the access to a target model $f: [0,1]^d\rightarrow \mathbb{R}^{|C|}$ and any target class $c^* \in C$, the attacker attempts to recover a representative point $x^*$ of the training data from the class $c^*$; $d$ represents the dimension of the model input; $C$ denotes the set of all class labels and $|C|$ is the size of the label set. For example, an attack on the face recognition classifier would try to recover the face image for a given identity based on the access to the classifier. 
\vspace{-2em}

\paragraph{Model knowledge.} The attacker's knowledge about a target model can take different forms: (i) Whitebox: complete access to all target model parameters; ii) Blackbox: access to the confidence scores output by the target model; and iii) Label-only: access to only the hard labels output by the model without the confidence scores. Our paper will focus on the label-only setting. Specifically, given the target network $f$, the attacker can query the target network at any input $x$ and obtain the corresponding predicted label $\hat{y}(x) = \arg \max_{c \in C } f_c(x)$.

\vspace{-1em}
\paragraph{Task Knowledge.} For the rest of the paper, we  assume that the attacker has knowledge about the task that the target model performs. This is a reasonable assumption, since this information is available for existing online models, or can be inferred from output labels.

\vspace{-1.1em}

\paragraph{Data Knowledge.} Since we assume that attackers know the task of the attacked model, it is reasonable to assume that they can gain access to a public dataset from a related distribution.
For example, if attackers know that the target model is trained to perform facial recognition, they can easily gather a public dataset by leveraging the existing open-sourced datasets or crawling data from the web. Throughout the paper, we assume that the public data and the private data do not share any classes (e.g., identities) in common.


\vspace{-1.1em}
\paragraph{Target models.} Our approach neither makes assumption on the target model architecture, nor requires the attacker to have any information about it. In other words, our approach is model-agnostic. We will empirically show in Section \ref{sec:experiments} that our \AlgName attack generalizes to a variety of models with different architectures and sizes.

\vspace{-1.1em}
\paragraph{Target labels.} The attack can be \emph{targeted}, when the goal is to find $n$ input images that maximize a set of $n$ \emph{predefined} labels, or \emph{untargeted}, when the goal is to find $n$ input images that maximize a set of \emph{any} $n$ labels. The proposed algorithm can apply to \emph{both} scenarios. In our evaluation, we will focus on the more challenging scenario, where the attack is targeted for $n$ specific labels.


\section{Algorithm Design}
\label{sec:Algorithm-design}
In this section, we will present the design of our proposed algorithm \AlgName. We will start by formulating the MI attack as an optimization problem. Then, we describe an algorithm to estimate the gradient of the MI optimization objective based only on predicted labels. We will rigorously characterize the alignment between the estimate and the true gradient for the special case of linear models and provide insights into the attack efficacy for deep, nonlinear models.

\subsection{Problem Formulation}

Without loss of generality, we state the attack problem formulation for a single target label and define $M_{c^*}: \mathbb{R}^d \rightarrow \mathbb{R}$ such that
\begin{equation}
    M_{c^*}(x) = f_{c^*}(x) - \max_{c \neq c^*} f_c(x), 
    \label{eqn:opt-over-x}
\end{equation}
where $c^*$ is the target label. $M_{c^*}(x)$ represents the logit (or confidence score) difference between the target class $c^*$ and the most likely label in the rest of the classes. Note that when $x$ is predicted into the target class (i.e., $c^*=\arg\max_{c\in C} f_c(x)$), $M_{c^*}(x)>0$. Clearly, the most representative input for the target class $c^*$ should be most distinguishable from all the other classes. Hence, we cast the MI problem into an optimization problem that seeks for the input that achieves maximum difference between the confidence for the target class and the highest confidence for the other classes:
\setlength{\abovedisplayskip}{4pt}
\begin{align}
    \arg\max_{x\in [0,1]^d}  M_{c^*}(x).
\end{align}

However, for images, $x$ usually lies in a high-dimensional continuous data space and optimizing over this space can easily get stuck in local minima that do not correspond to any meaningful images.
To resolve this issue, we leverage the idea in~\cite{yang2019adversarial, zhang2020secret,chen2021knowledgeenriched} and optimize over a more semantically meaningful latent space. This is done by using a public dataset to train GAN models and then optimizing over the input to the GAN generator. Denote the publicly trained generator by $G(z)$, where $z \in \mathbb{R}^{d'}$ and $d' < d$. Now, the MI optimization problem can be updated to reflect the change of optimizing $z$ rather than $x$  as follows:
\begin{equation}
   \arg\max_{z\in \mathbb{R}^d} M_{c^*}(G(z)).
\end{equation}
Unlike the whitebox setting, we cannot directly optimize $M_{c^*}(G(z))$ using gradients as we do not have access to the model parameters $f$. Moreover, it is also not possible to apply zero-order optimization algorithms, as they require access to the confidence scores output by the model.

\subsection{\AlgName Algorithm}
 
The intuition behind our algorithm is that the farther a point is from the decision boundary of a class, the more representative this point becomes to the class. Thus, the centroid of any class should be its good representative. 
Inspired by this, we design an algorithm which tries to gradually move away from the decision boundary. In a high level, our algorithm proceeds by first sampling points over a sphere and then querying their labels. Intuitively, the points that are not predicted into the target class represent the directions that we want to move away from. Hence, we take an average over those points and move in the direction opposite to the average. If all the points are predicted into the target class, then we will increase the radius.



Let $\text{sign}(\cdot)$ be a function that returns $1$ if the input is positive and $-1$, otherwise. We define $\Phi_{c^*}: \mathbb{R}^d \rightarrow \{-1,0\}$:
\begin{align}
\small
  \Phi_{c^*}(z) &= \frac{\text{sign} (M_{c^*}(z)) - 1}{2} \\
  &= \left\{\begin{array}{ll}
     & 0, \text{ if } c^* = \arg \underset{c \in C}{\max} f_c(G(z))\\
     & -1, \text{ otherwise.}
\end{array}\right.
\end{align}
Essentially, $ \Phi_{c^*}(z)$ marks points that are not predicted into the target class. Then, we define our gradient estimator as
\begin{equation}
    \widehat{M_{c^*}}(z, R) = \frac{1}{N}\sum_{n=1}^N \Phi_{c^*}(z+Ru_n)u_n,
    \label{eqn:gradient-estimator}
\end{equation}
where $u_n$ is a uniformly random point sampled over a $d'$-dimensional sphere with radius $R$ and $N$ is the number of points sampled on the sphere. Note that $\widehat{M_{c^*}}(z, R)$ can be calculated in the label-only setting as it only requires the knowledge of predicted label of the sampled points. We will then use $\widehat{M_{c'}}(z, R)$ to update $z$:
\begin{equation}
    z\leftarrow z+ \alpha \widehat{M_{c^*}}(z,R),
    \label{eqn:update-point}
\end{equation}
where $\alpha$ is the update step size. It can be either a fixed value or a function of the current radius $R$. When all points sampled from the sphere of the current radius are predicted into the target class, i.e., $ \Phi_{c^*}(z+Ru_n)=0$ for all $n=1,\ldots,N$, then we increase the radius and alternate between estimating $\widehat{M_{c'}}(z, R)$ using Eq. (\ref{eqn:gradient-estimator}) and performing update with Eq. (\ref{eqn:update-point}) at the new radius.

The pseudo-code of \AlgName is provided in Algorithm~\ref{alg:main}. \AlgName starts with the initial point correctly classified as the target class. To ensure this, images are sampled from the GAN until a point belonging to the target class is generated. Note that the initial point, although classified into the target class, is almost \emph{never} a representative point for the target class (see more examples in Fig.~\ref{fig:radius}). The radius of the sphere is initialized to a reasonably small value. Then, the algorithm will try to move away from the decision boundary iteratively. At each iteration, we sample $N$ points on the sphere with radius $R$ centered at the current point and query their labels from the target model. If all the points are classified into the target class, the radius will be enlarged; otherwise, we estimate $\widehat{M_{c'}}(z, R)$ using Eq. (\ref{eqn:gradient-estimator}) and update $z$ according to Eq. (\ref{eqn:update-point}). Note that the update is reverted if the new point $z$ lies outside the target class. In that case, we will resample the points on the sphere and compute a new update. The algorithm will be halted when it is not possible to find a larger sphere such that all the samples on that sphere fall into the target class. The output of the algorithm is a point ($z^*$) with the largest sphere that can fit into the target class. This indicates that the point is the farthest from the boundary. We will use this point to evaluate the attack.





\begin{algorithm}[ht]
\SetAlgoLined
\SetKwInOut{Input}{input}
\SetKwInOut{Output}{output}
\SetKwInput{Ensure}{ensure}
\Input{Target model's hard-label prediction $\hat{y}$ ; target class $c^*$, number of samples $N$; number of maximum iterations $maxIters$; initial sphere sampling radius $R_0$; radius multiplier $\gamma$; data point learning rate $\alpha$}
\Output{Representative sample $z^*$ for $c^*$.}
\Ensure{A sample $z$ in the target class $c^{*}$ by repeatedly sampling from the GAN's latent space.} 
Set $R \leftarrow R_0$.\\ 
Set $iters \leftarrow 0$.\\
Set $points \leftarrow vector(N)$\\
\While {iters $<$ maxIters} { 
    $points \leftarrow \text{random N points on a sphere r=R}$\\
    \small{\tcp{ Check if all sampled points are in target class.}}
    \eIf{ $points$ in $c^*$}{
        \small{\tcp{Update radius and current best point }}
        $R \leftarrow R\times \gamma$ .\\
        $z^* \leftarrow z$ .\\
        $iters \leftarrow 0 $.\\
    }
    {
        Compute $\widehat{M_{c^*}}(z,R)$ via Eq. (\ref{eqn:gradient-estimator})\\
        $z_\text{new}\leftarrow$ the RHS of Eq. (\ref{eqn:update-point})\\
        \If{if  $\hat{y}(z_\text{new})=c^*$} {
            $z\leftarrow z_\text{new}$
        }
    }
}
\caption{ \AlgName Decision-Based Zero Order Optimization Algorithm.}
\label{alg:main}
\end{algorithm}

\subsection{Attack Justification}
\label{sec:analysis}

As our gradient estimator $\widehat{M_{c^*}}(z, R)$ repels non-target-class points, intuitively, it points towards the direction that increases the target class' likelihood. We provide a theorem that characterizes the alignment between the proposed estimator and the true gradient $\nabla M_{c^*}\left(z\right)$ for special cases of linear classification models (e.g., logistic regression).



\begin{theorem}
\label{thm:linear}
Assume $f$ has a linear classification model. Let $z$ be an arbitrary point within the target class, i.e. $M_{c^*}(z)>0$. Then, the cosine of the angle between $\mathbb{E}[\widehat{M_{c^*}}(z, R)]$ and $\nabla M_{c^*}\left(z\right)$ is bounded by 
\begin{align}
\cos \angle\left(\mathbb{E}\left[\widehat{\nabla M_{c^*}}\left(z, R\right)\right], \nabla M_{c^*}\left(z\right)\right) \\
\geq 1- \mathcal{O}\bigg(\frac{M_{c^*}\left(z\right)^2  (d-1)^2}{\delta^2 R^2 \left\| \nabla M_{c^*}\left(z\right) \right\|_{2}^2}\bigg).
\end{align}

Therefore, with increasing radius $R$,
$$
\lim _{R \rightarrow \infty} \cos \angle\left(\mathbb{E}\left[\widehat{M_{c^*}}\left(z, R \right)\right], \nabla M_{c^{\star}}\left(z\right)\right)=1,
$$
which tells that the estimator is asymptotically unbiased for gradient estimation.
\end{theorem}
The proof is provided in the Appendix 1.1. Theorem~\ref{thm:linear} shows that as long as $R$ is large enough, the gradient estimator aligns well with the actual gradient. 

For the deep learning model with bounded nonlinearity, we can also derive the bound for the cosine of the angle between the estimate and the true gradient: $\cos \angle\left(\mathbb{E}\left[\widehat{\nabla M_{c^*}}\left(z, R\right)\right], \nabla M_{c^*}\left(z\right)\right) 
\geq 1- \mathcal{O}\bigg(\frac{\left[M_{c^*}\left(z\right)^2 + L^2\delta^4 R^4  + 4M_c(z)\delta^2LR^2 \right] (d-1)^2}{\delta^2 R^2 \left\| \nabla M_{c^*}\left(z\right) \right\|_{2}^2}\bigg)$, where $L$ characterizes the level of nonlinearity. A more formal statement of the result will be provided in Appendix. The result shows that with increasing $R$ the estimated gradient will align with the true gradient. However, after a certain point of inflection, increasing radius will only decrease the accuracy of the estimate. This occurrence is actually well presented in our experiments in Section \ref{sec:experiments}, which shows the ``sweet point" for reaching the radius R that maximizes the success rate. 

\section{Evaluation}
\label{sec:experiments}
Our evaluation aims to answer the following questions: (1) Can \AlgName successfully attack deep nets with different architectures and trained on different datasets? (2) How many queries does \AlgName require to perform a successful attack? (3) How does the distributional shift between private data and public data affect the attack performance? (4) How sensitive is \AlgName to the initialization and the sphere radius? In the main text, we will focus on a canonical application--face recognition--as our attack target. We will leave experiments on other applications to the Appendix.



\subsection{Experimental Setup}
\paragraph{Datasets.} We experiment on three different face recognition datasets: CelebA \cite{liu2015deep}, Facescrub \cite{ng2014data}, and Pubfig83 \cite{pinto2011scaling}.
Similar to \cite{zhang2020secret,chen2021knowledgeenriched, yang2019adversarial} we crop the images of all datasets to the center and resize them to $64x64$. We split the identities into the public domain (which we train our GAN on), and the private domain (which we will train target models on). \emph{There are no overlapping identities between the public and the private domain. }This means that the attacker has zero knowledge about the identities in the private domain. We then perform the attack on the classifier that is trained on the private domain. The details about each dataset are shown in Table~\ref{tab:datasets-details}.
To study the impact of a large distributional shift between private and public domain on the attack performance, we use the FFHQ dataset~\cite{karras2019style} as our public domain to train the GAN and the aforementioned three datasets as the private domains.

\begin{table}[ht!]
\centering
\scalebox{0.75}{
\begin{tabular}{lccccc}
\toprule
Dataset & \#Images & \#Total id & \#Public id & \#Private id & \# Target id\\
\midrule 
CelebA & 202,599 & 10,177 & 9,177 & 1,000 & 300 \\
Pubfig83 & 13,600 & 83 & 33 & 50 & 50 \\ 
Facescrub & 106,863& 530 & 330 & 200 & 200 \\ 
\bottomrule
\end{tabular}
}
\caption{\label{tab:datasets-details} Details on how we split datasets in evaluations into the public and the private domains.   }
\end{table}

\paragraph{Target Models.} In addition to evaluating our attack on a range of datasets, we also evaluate our attack on different models with a variety of architectures. To provide consistent results with the previous work, we use the same model architectures used in the state-of-the-art MI attack~\cite{chen2021knowledgeenriched}: (1) face.evoLve adpated from \cite{cheng2017know}; (2) ResNet-152 adapted from \cite{he2016deep}; and (3) VGG16 adapted from \cite{simonyan2014very}.
\vspace{-1em}
\paragraph{Evaluation Protocol.}
For all evaluations with \AlgName, we perform a targeted attack as it is a more challenging setting compared to untargeted attack. Following  \cite{zhang2020secret, chen2021knowledgeenriched}, we use the \emph{attack accuracy} to measure the attack performance. The attack accuracy is based on an evaluation classifier, which predicts the identities for reconstructed face images and is a proxy for human judge. Specifically, the attack accuracy is calculated by the ratio of the number of reconstructed images that are correctly classified into the corresponding target classes over the total number of reconstructed images. As the evaluation classifier reflects a human judge, it should have high performance. At the same time, it should be different from the attacked target models to avoid some semantically meaningless reconstructed images that overfits to the target models to be considered as good reconstructions. 

\vspace{-1em}
\paragraph{Hyperparameters.} We manually fine-tuned the hyperparameters of \AlgName in our evaluations. We found empirically that the best initial radius $R_0$ is $2$, the radius expansion coefficient $\gamma$ is 1.3, and the step size $\alpha_t$ = min($R/3$, 3). We choose N, the number of sampled points on a sphere, to be 32 unless otherwise specified. maxIters is chosen to be 1000, i.e., \AlgName terminates when more than 1000 iterations have passed for a certain $R$ without having all points on sphere to be classified as the target class.





\paragraph{Baselines.}
Since this is the first work that provides a solution to label-only MI attacks, we have no baselines to evaluate against.
We opt to evaluate against whitebox and blackbox attacks in which the attacker has a \emph{greater} advantage in terms of additional knowledge about target models. 
To ensure fair comparison, we apply all baselines over the same set of target identities for each dataset and the same target models. Then we evaluate attack accuracy against the same evaluation classifiers. Two of our baselines are white box attacks, including Generative Model Inversion (\textbf{GMI}) \cite{zhang2020secret}, which is the first MI attack algorithm against deep nets, and Knowledge-Enriched Distributional Model Inversion attack (\textbf{KED-MI}) \cite{chen2021knowledgeenriched}, which provides the currently state-of-the-art performance for whitebox MI. The GAN models in GMI is set to be the same as the GAN in our attack. KED-MI relies on the access to target model parameters in training the GAN models. However, we cannot access such information and train the same GAN in our setting. We also employ a blackbox attack~\cite{yang2019adversarial}, referred to as the learning-based model inversion (\textbf{LB-MI}) as one of our baselines. LB-MI builds an inversion model that learns to reconstruct images from the soft-labels produced by the target model. To reconstruct the most representative image for a given identity, we feed a one-hot encoding for that identity at the input of the inversion model and receive the output.

\subsection{Results}

\paragraph{Performance on Different Datasets.}

We compare \AlgName to whitebox and blackbox methods on the three different face datasets. We use FaceNet64 as the target model across all datasets. For each dataset, the GAN models are trained on its public identities, and target models are trained on the private identities. 
Table \ref{tab:diff-datasets} shows that our approach considerably outperforms both the whitebox GMI attack and the blackbox attack on all datasets.
Further, our method surpasses the state-of-the-art whitebox KED-MI attack on Pubfig83 and achieves a close attack accuracy on the CelebA dataset.
On the other hand, we fall behind by 15\% on the Facescrub dataset. It is worth noting that the outcome of this experiment implies that there is still a considerable potential for development in the other threat models in MI attacks, particularly blackbox attacks (which perform poorly with respect to the other threat models). The reason why GMI performs poorly even with the whitebox knowledge is that it optimizes the likelihood of only the synthesized data point without considering the neighborhood of the point. Hence, it is possible that optimization gets stuck in a sharp local maximum that does not represent the class.  On the other hand, both \AlgName and KED-MI explicitly finds a neighborhood with high likelihood, which turns out to be crucial to produce representative points and enhance attack performance. It is worth noting that the blackbox attack, although leveraging more knowledge about target model than our attack, consistently achieves the worst performance. Compared to the other attacks, the blackbox attack utilizes a very different idea for distilling knowledge from public datasets. It uses the public data to train the inversion model whereas the other attacks all train GAN models on the public data. The results suggest that GANs are more effective in distilling public knowledge than an inversion model. So a potential way to improve blackbox attack is to regularize the synthesized images via GAN.

\begin{table}[t!]
\small
\centering
\scalebox{\tablescale}{
\begin{tabular}{lcccc}
\toprule
\multirow{2}{*}{Dataset} & \multicolumn{2}{c}{[Whitebox]} & [Blackbox] & [Label-only]\\
&  GMI & KED-MI & LB-MI & \textbf{\AlgName}\\
\midrule 
CelebA & 32.00\% & 82.00\% & 1.67\% & 75.67\% \\
Pubfig83 & 24.00\% & 62.00\% & 2.00\% & 66.00\% \\
Facescrub & 19.00\% & 48.00\%  & 0.50\% & 35.68\% \\

\bottomrule

\end{tabular}
}
\caption{\label{tab:diff-datasets} Attack performance comparison various datasets.  }
\end{table}




\paragraph{Performance on Different Models.} We also evaluate our attack on multiple different models trained on the same dataset (CelebA). This experiment is intended to test whether our approach can generalize to different model architectures.
Table~\ref{tab:celebA-models} shows that \AlgName indeed continues to perform well on a variety of target model architectures. In particular, \AlgName outperforms GMI and the blackbox attack by a substantial margin for all model architectures. As we can see, the attack accuracy is $2x-4x$ that of GMI attack, while the blackbox attack continues to have $< 2\%$ accuracy.
Additionally, our performance on all model architectures is comparable to that of the state-of-the-art whitebox. Similar to other attacks, our attack becomes more successful when the target model has higher predictive power.

\begin{table}[ht]
\small
\centering
\scalebox{\tablescale}{
\begin{tabular}{ccccc}
\toprule
\multirow{2}{*}{Model Archt.} & \multicolumn{2}{c}{[Whitebox]} & [Blackbox] & [Label-only]\\
&  GMI & KED-MI & LB-MI & \textbf{\AlgName}\\
\midrule 
FaceNet64 & 32.00\% & 82.00\% & 1.67\% & 75.67\% \\
IR152  & 26.00\% & 83.00\% & 0.33\% & 72.00\% \\
VGG16  & 15.00\% & 69.00\%  & 1.33\% & 63.33\% \\
\bottomrule

\end{tabular}
}
\caption{\label{tab:celebA-models} Attack performance comparison on different model architectures trained on the CelebA dataset.  }
\end{table}

\paragraph{Cross-Dataset Evaluation.}

\vspace{-2em}
\begin{table}[ht]
\small
\centering
\scalebox{0.85}{
\begin{tabular}{lcccc}
\toprule
\multirow{2}{*}{Public$\rightarrow$Private} & \multicolumn{2}{c}{[Whitebox]} & [Blackbox] & [Label-only]\\
&  GMI & KED-MI & LB-MI & \textbf{\AlgName}\\
\midrule 
FFHQ$\rightarrow$CelebA & 9.00\% & 48.33\% & 0.67\% & 46.00\% \\
FFHQ$\rightarrow$Pubfig83 & 28.00\% & 88.00\% & 4.00\% & 80.00\% \\
FFHQ$\rightarrow$Facescrub & 12.00\% & 60.00\%  & .015\% & 39.20\% \\
\bottomrule

\end{tabular}
}
\caption{\label{tab:cross-datasets} Performance comparison when there is a large distribution shift between public and private data.  }
\end{table}

In prior experiments, we assumed that the attacker had access to public data with low distributional shift with the private data.
This is because both public and private domains are derived from the same dataset. It is important to consider a more pragmatic scenario, in which the attacker has access to only public data that have a larger distributional shift. To investigate this scenario, we perform an experiment in which we use the FFHQ dataset as public data.

As shown in Table \ref{tab:cross-datasets}, the accuracy indeed decreases significantly for the CelebA dataset when we utilize FFHQ as our public dataset. Interestingly, the attack accuracy for Pubfig83 and Facescrub datasets has increased. The rationale for this performance boost is that Pubfig83 and Facescrub datasets have just 33 and 330 identities in their public distributions, respectively, as shown in Table~\ref{tab:datasets-details}. This means that the GAN models trained on these datasets would lack the ability to generalize and thus, produce bad results. Therefore, the ability of GAN models to generalize to the large number of identities in FFHQ compensate for the distributional shift and consequently, the results improve. On the other hand, the CelebA dataset has a rather significant number of public identities (9177 identities). Thus, the GAN is already capable of generalizing across different identities, and the performance increase associated with generalizing on a more varied dataset is insufficient to compensate for the performance reduction associated with distributional shift. The takeaway from this experiment is that having a large, diverse public data for distilling a distributional prior is crucial to MI attack performance.

\paragraph{Limited Query Budget.}
We investigate the performance of our attack at various query budgets. In practice, some online models, such as Google's cloud vision API, limit the number of queries per minute, others may ban users if they identify an unusually high volume of queries.  Due to the fact that some attack scenarios restrict the amount of queries that may be sent to the target model, it is important to investigate the impact of this restriction on the attack performance. 

This restriction has not been addressed in prior works that conduct whitebox MI attacks in the literature. This is because the attacker, by definition, has complete access to the model parameters and can thus create an offline copy of the model, and then proceed with the attack offline with unlimited queries. However, for blackbox attacks in general (including label-only attack), the user cannot copy model parameters to an offline model. As a result, the query budget may become a constraint.

Fig.\ref{fig:budget-graphs} (a) demonstrates how \AlgName performs under different query budgets. We see that the attack accuracy increases exponentially by increasing the query budget. This is true until we hit some query budget, then attack accuracy starts decreasing again. We will provide some insights on it in later in the paper. 
For all datasets examined in this paper, recovering a representative image to a private class requires from $10k$ to $16k$ queries to the model, which is very reasonable. 

The attacker should also be concerned when choosing the hyperparameter $N$ under limited query budgets.
Choosing large $N$ would increase the number of sampled points on sphere, and produce a better estimator for our update direction. On the other hand, for a fixed query budget, increasing $N$ means decreasing the number of possible iterations in the attack. 
We conducted experiments to show this trade-off between spending queries to get better gradient estimator per iteration vs using queries to apply more iterations. Fig. 3 (b), (c), and (d) indicate that, for small query budgets, \AlgName performs slightly better when spending query budget on increasing the number of iterations, instead of increasing $N$. However, for sufficiently large query budget, increasing $N$ produces better results.



\begin{figure}[ht]
  \centering
  \includegraphics [width=0.9999\linewidth]{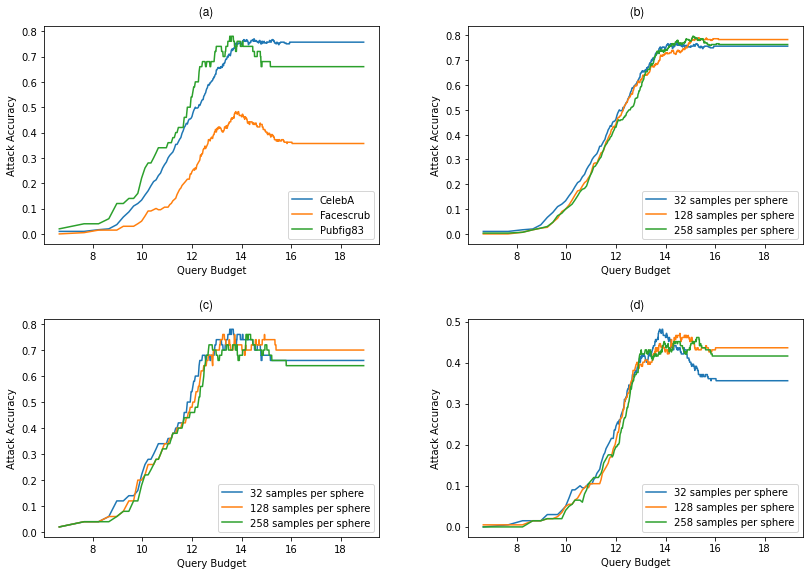}
  \caption{Attack accuracy of \AlgName under different query budgets. (a) compares different datasets. (b), (c), and (d) compare different sampling strategies for CelebA, pubfig, and Facescrub respectively. Query Budget is shown in a base-2 log for convenience. }
  \label{fig:budget-graphs}
\end{figure}

%
%

\vspace{-1em}

\paragraph{Analyzing \AlgName.} A qualitative analysis for our \AlgName can be seen in Fig. \ref{fig:radius}. It is noticeable that the first generated image at the beginning of the attack is not a good representative for the target class. The progression of the image towards the groundtruth images is clearly seen with the increase of $R$ .

\begin{figure}[ht]
  \centering
  \includegraphics [scale =0.6]{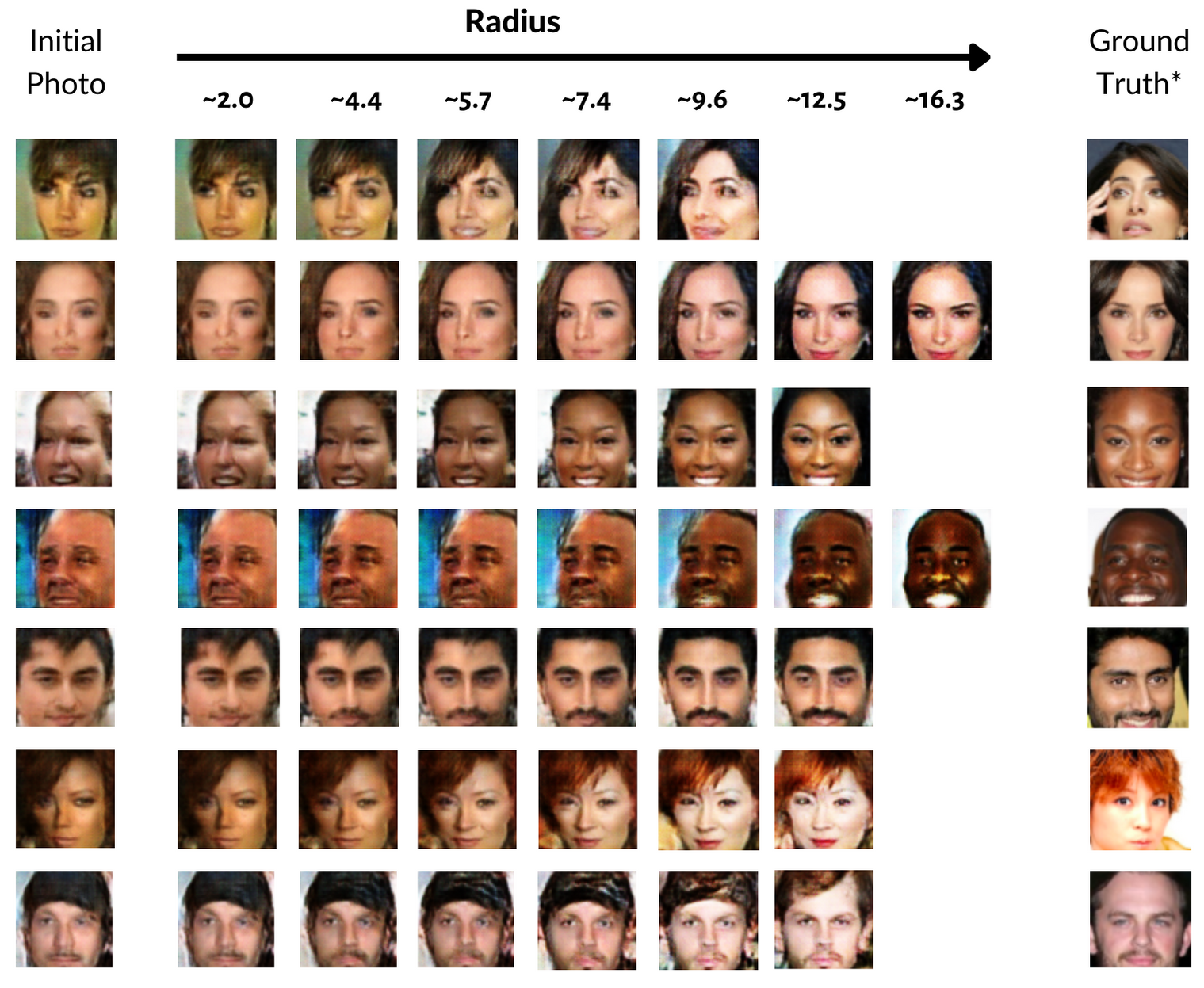}
  \caption{\AlgName's progression along each radius from the first random initial point until the algorithm's termination. }
  \label{fig:radius}
\end{figure}


Below, we provide some quantitative analysis. Table \ref{tab:analaysis-1} analyzes the intermediate steps when attacking FaceNet64 model on CelebA dataset. We say the attack reached a radius $R$ when it finds a  center point, for which all points sampled on a sphere with radius $R$ lie in the correct target class.
We report for each reached radius during the attack the following measurements: (i) the percentage of the target identities that successfully reach it (column: labels \% ); (ii) the minimum, maximum, average number of iterations required to reach it; and (iii) the attack success accuracy of the points that reached it (column: success \%).

As we can see in the "labels \%" column, \AlgName is able to increase $R$ multiple times for all target identities. In fact, all target identities had their $R$ increased by \AlgName at least 5 times. This shows the effectiveness of our algorithm to repel away from the boundary and get closer to the center of the class (which is our goal). 

Another interesting observation is that the bigger the radius is, the higher the attack accuracy we get. 
This is true until reaching a certain radius size then the accuracy starts dropping. As suggested by our theoretical analysis, at high radii, the gradient estimator becomes erroneous; hence, following the direction of gradient estimator will decrease the attack accuracy at those radii. This observation is consistent for all our conducted experiments regardless the model or dataset.
Unfortunately, since the attacker does not have any ground truth images of the target class (or an evaluator classifier), it is not possible to decide what is the best radius that the algorithm should stop at. Nevertheless, as seen in the table \ref{tab:analaysis-1}, the number of identities that reached those radii is very low and their contribution to the final attack accuracy is small. Additionally, our stopping criterion for the algorithm empirically provides close results to the best radius. For this experiment particularly, we were able to get 75.67\% which is close to stopping at the best radius. 





\begin{table}[h]
\centering
\scalebox{0.85}{

\small{
\begin{tabular}{lccccc}
\toprule
Radius &  labels \%  & min iters & max iters & avg iters &  success \%\\

\midrule 
 2.00 &  100.00\% &  0 &  191 &  37.29 &  23.00\% \\ 
 2.60 &  100.00\% &  0 &  246 &  63.20 &  30.33\% \\ 
 3.38 &  100.00\% &  0 &  374 &  103.43 &  45.00\% \\ 
 4.39 &  100.00\% &  2 &  627 &  156.65 &  56.00\% \\ 
 5.71 &  100.00\% &  28 &  947 &  230.64 &  63.67\% \\ 
 7.43 &  100.00\% &  53 &  1721 &  336.38 &  71.67\% \\ 
 9.65 &  97.00\% &  89 &  1899 &  502.60 &  77.66\% \\ 
 12.55 &  71.00\% &  141 &  1909 &  746.60 &  80.28\% \\ 
 16.31 &  20.33\% &  298 &  1823 &  939.90 &  70.49\% \\ 
 21.21 &  1.67\% &  492 &  1875 &  1122.00 &  60.00\% \\ 
 27.57 &  0.67\% &  660 &  728 &  694.00 &  50.00\% \\ 
 35.84 &  0.33\% &  877 &  877 &  877.00 &  0.00\% \\ 
 \bottomrule
 
\end{tabular}
}}
\caption{\label{tab:analaysis-1} Analysis on the intermediate steps of our algorithm.}
\end{table}

\begin{figure}[t!]
  \centering
  \includegraphics [width=.85\linewidth]{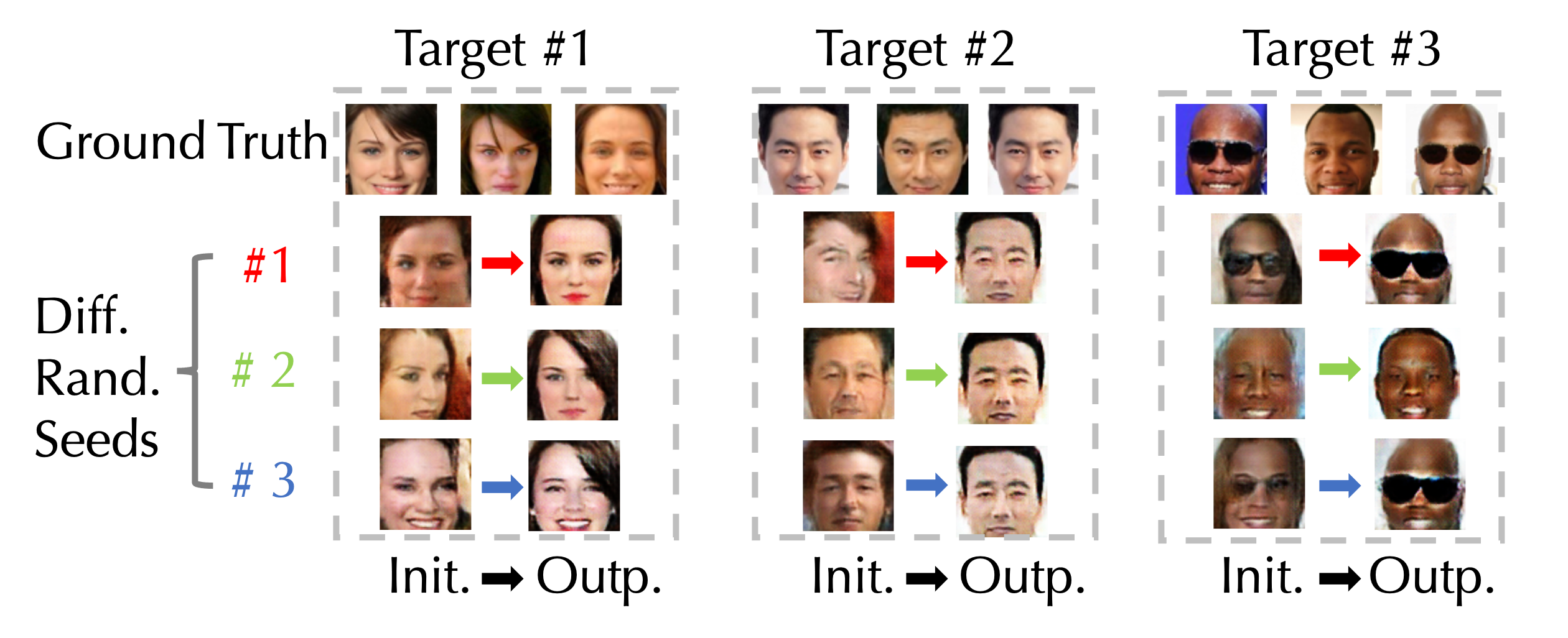}
  \caption{Ablation study of initialization of \AlgName. ``Init" is the initial sampled point that we start our attack from. ``Outp" is the final output of the attack. }
  \label{fig:init}
\end{figure}

\paragraph{Effect of Random Sampling}
Due to the fact that \AlgName starts from an initial random point, we conducted an experiment to show whether different initial points would affect the algorithm outcome. 
We started three attacks with different random seeds on a FaceNet64 model trained on the CelebA dataset. The accuracy is 75.67\%,  76.33\%, and 75.67\% respectively. This shows that the random initial point has little effect quantitatively on our algorithm. Fig.\ref{fig:init} demonstrates our quantitative results, where we can observe that even under different initial points, the output of the algorithm is close to the ground truth images.

\section{Conclusion}
We presented a novel algorithm to perform the first label-only MI attack. Experiments have showed the effectiveness of our approach on different datasets and model architectures. Interestingly, the approach provides comparable results with the state-of-the-art whitebox attacks and outperforms all the other baselines despite the fact that they make stronger assumption about the attacker knowledge. 
As future work, the closeness of the results in multiple experiments between the label-only attack and the state-of-the-art whitebox attack indicates that there may still be room for improvement for whitebox attack. Similarly, the blackbox baseline attack underperforms our label-only attack by a huge margin although it can access more fine-grained model output than our label-only attack. Theoretically, it should be an upper bound for our performance. 

\clearpage
{\small
\bibliographystyle{ieee_fullname}
\bibliography{paper}
}

\appendix
\paragraph{\textbf{Appendix}}

\section{Bounding Estimated Gradient}

The following proof technique was inspired by \cite{chen2020hopskipjumpattack}. However, in our case the major difference is that our initial point $z$ is an arbitrary point within the target class instead of being near the decision boundary, the fact which the authors \cite{chen2020hopskipjumpattack} used to assume $M_{c^*}\left(z\right) = 0$.
Furthermore, our goal is to move towards the centroid of a target class not to reach the boundary, which is completely opposite direction than ours. 
In summary, by exploiting the Taylor's theorem and the Minkowski inequality for expectations, we bounded the cosine angle between our gradient estimator and the true gradient, which we present for two different attack settings.

\subsection{Linear Case}
\label{sec:linearcase}

\begin{theorem}
\label{thm:linear}
Assume $f$ has a linear classification model. Let $z$ be an arbitrary point within the target class, i.e. $M_{c^*}(z)>0$. Then, the cosine of the angle between $\mathbb{E}[\widehat{M_{c^*}}(z, R)]$ and $\nabla M_{c^*}\left(z\right)$ is bounded by 
\begin{multline}
\cos \angle\left(\mathbb{E}\left[\widehat{\nabla M_{c^*}}\left(z, R\right)\right], \nabla M_{c^*}\left(z\right)\right) \\
\geq 1- \mathcal{O}\bigg(\frac{M_{c^*}\left(z\right)^2  (d-1)^2}{\delta^2 R^2 \left\| \nabla M_{c^*}\left(z\right) \right\|_{2}^2}\bigg).
\end{multline}
Therefore, with increasing radius $R$,
\begin{equation}
\lim _{R \rightarrow \infty} \cos \angle\left(\mathbb{E}\left[\widehat{M_{c^*}}\left(z, R \right)\right], \nabla M_{c^{\star}}\left(z\right)\right)=1,
\end{equation}
which tells that the estimator is asymptotically unbiased for gradient estimation.
\end{theorem}

\begin{proof}
Let $Ru$ be a random vector uniformly distributed on the sphere, where $R > 0$ is a radius of the sphere. By Taylor's theorem, for any $\delta \in(0,1)$, we have that
\begin{equation}
M_{c^*}\left(z+\delta Ru\right)=M_{c^*}\left(z\right) + \delta \nabla M_{c^*}\left(z\right)^{T} Ru.
\end{equation}
Recall that 
$$M_{c^*}(z) > 0$$
and let $w:=\frac{M_{c^*}(z)}{\delta R} .$ \\
For the case $\nabla M_{c^*}\left(z\right)^{T} u>w$, using Taylor series expansion and the fact that

$$ \delta\nabla M_{c^*}(z)^{T} Ru > w \delta R = M_{c^*}(z) > 0, $$
we derive that $M_{c^*}\left(z+\delta Ru\right) > M_{c^*}(z)$. \\
Similarly, for the case $\nabla M_{c^*}\left(x_{t}\right)^{T} u <- w$, using Taylor expansion and the fact that
$$ \delta \nabla M_{c^*}\left(z\right)^{T} Ru < -w\delta R = - M_{c^*}(z) < 0,$$
 we have that $M_{c^*}\left(z+\delta Ru\right) < M_{c^*}(z)$. \\
 Therefore, from these two cases, we arrive at
\begin{align}\label{eq:linearphi}
\phi_{c^*}\left(z+\delta Ru\right)=\left\{\begin{array}{l}
0 \text { if } \nabla M_{c^*}\left(z\right)^{T} Ru>w \\
-1 \text { if } \nabla M_{c^*}\left(z\right)^{T} Ru<-w.
\end{array}\right.
\end{align}
We define  $v_{1}=\nabla S\left(z\right) /\left\|\nabla M_{c^*}\left(z\right)\right\|_{2}, v_{2}, \ldots, v_{d} $ by expanding the vector $\nabla M_{c^*}\left(z\right)$ to orthogonal bases in $\mathbb{R}^{d}$. Then, we can write a random vector $Ru=\sum_{i=1}^{d} \beta_{i} v_{i}$, where $\beta$ is uniformly distributed on the sphere of radius $R$. We construct an upper cap $E_{1}:=$ $\left\{\nabla M_{c^*}\left(z\right)^{T} Ru>w\right\}$, the annulus $E_{2}:=\left\{\left|\nabla M_{c^*}\left(z\right)^{T} Ru\right|<w. \right \}$, and the lower cap $E_{3}:=\left\{\nabla M_{c^*}\left(z\right)^{T} Ru<-w\right\} .$ Let $p:=\p\left(E_{2}\right)$ be the probability of event $E_{2}$, then
$\p\left(E_{1}\right)=\p\left(E_{3}\right)=(1-p) / 2 .$ 
For any $i \neq 1$ by symmetry:
\begin{align}
\e\left[\beta_{i} \mid E_{1}\right]=\e\left[\beta_{i} \mid E_{3}\right]=0.
\end{align}
Then, the expected value of the estimator becomes

\begin{align}
\begin{split}\label{eq:expect}
\mathbb{E}\left[\phi_{c^*}\left(z  + \delta Ru\right) u\right]= p \cdot\left(\mathbb{E}\left[\phi_{c^*}\left(z+\delta Ru\right) u \mid E_{2}\right]\right) \\
-\frac{1}{2} p \cdot \left(\mathbb{E}\left[\beta_{1} v_{1} \mid E_{1}\right] + \mathbb{E}\left[-\beta_{1} v_{1} \mid E_{3}\right]\right) \\
+\frac{1}{2}\mathbb{E}\left[\beta_{1} v_{1} \mid E_{1}\right]+\frac{1}{2}\mathbb{E}\left[-\beta_{1}  v_{1} \mid E_{3}\right].
\end{split} 
\end{align}
Now, we can bound the difference between $\mathbb{E}\left[\left|\beta_{1}\right| v_{1}\right]=\frac{\mathbb{E}\left[\left|\beta_{1}\right|\right]}{\left\|\nabla M_{c^*}\left(z\right)\right\|_{2}} \nabla M_{c^*}\left(z\right)$ and $\mathbb{E}\left[\phi_{c^*}\left(z+\delta Ru\right) u\right]$ using  Eq. \ref{eq:expect}:

\begin{align}
\begin{split}\label{eq:expbound}
&\left\|\mathbb{E}\left[\phi_{c^*}\left(z+\delta Ru\right) u\right]  -  \mathbb{E}\left[\left|\beta_{1}\right| v_{1}\right]\right\|_{2} \\
& \leq \left\| p \cdot\left(\mathbb{E}\left[\phi_{c^*}\left(z+\delta Ru\right) u \mid E_{2}\right]\right) \right.\\
&- \left. \frac{1}{2} p \cdot \left(\mathbb{E}\left[\beta_{1} v_{1} \mid E_{1}\right] + \mathbb{E}\left[-\beta_{1} v_{1} \mid E_{3}\right]\right) \right\|_{2} \\
&\leq p \cdot (R + \frac{1}{2}R + \frac{1}{2}R) = 2Rp.
\end{split}
\end{align}
In Ineq. \ref{eq:expbound}, we first substitute the LHS with  $\mathbb{E}\left[\left|\beta_{1}\right| v_{1}\right]=\frac{\mathbb{E}\left[\left|\beta_{1}\right|\right]}{\left\|\nabla M_{c^*}\left(z\right)\right\|_{2}} \nabla M_{c^*}\left(z\right)$, then square both sides, and lastly divide by  $\mathbb{E}\left[\left|\beta_{1}\right|\right]^2$ to derive the following

\begin{align}
\left\|\frac{\mathbb{E}\left[\phi_{c^*}\left(z+\delta Ru\right) u\right]}{\mathbb{E}\left[\left|\beta_{1}\right|\right]}-\frac{\nabla M_{c^*}\left(z\right)}{\left\|\nabla M_{c^*}\left(z\right)\right\|_{2}}\right\|_{2}^2 \leq \left(\frac{2Rp}{\mathbb{E}\left[\left|\beta_1 \right| \right]}\right)^2.
\end{align}

Using the property of the angle between two vectors, we obtain the cosine inequality:
$$
\cos \angle\left(\mathbb{E}\left[\phi_{c^*}\left(z+\delta Ru\right) u\right], \nabla M_{c^*}\left(z\right)\right) \geq 1-\frac{1}{2}\left(\frac{2 Rp}{\mathbb{E}\left[\left|\beta_{1}\right|\right]}\right)^{2}.
$$
\begin{equation} \label{eq:cos}
\end{equation}
The probability $p$ can be bounded using the fact that $\left\langle\frac{\nabla M_{c^*}\left(z\right)}{\left\|\nabla M_{c^*}\left(z\right)\right\|_{2}}, u\right\rangle^{2}$ is a Beta distribution $\mathcal{B}\left(\frac{1}{2}, \frac{d-1}{2}\right)$ :

\begin{align}
\begin{split}\label{eq:pbound}
p &=\mathbb{P}\left(\left\langle\frac{\nabla M_{c^*}\left(z\right)}{\left\|\nabla M_{c^*}\left(z\right)\right\|_{2}}, u\right\rangle^{2} \leq \frac{w^{2}}{\left\|\nabla M_{c^*}\left(z\right)\right\|_{2}^{2}}\right) \\
& \leq \frac{2 w}{\mathcal{B}\left(\frac{1}{2}, \frac{d-1}{2}\right)\left\|\nabla M_{c^*}\left(z\right)\right\|_{2}}.
\end{split}
\end{align}
Substituting the Bound \ref{eq:pbound} into Eq. \ref{eq:cos}, we receive a bound for the cosine

\begin{align}
\begin{split}\label{eq:cosbound}
\cos \angle\left(\mathbb{E}\left[\phi_{c^*}\left(z+\delta Ru\right) u\right], \nabla M_{c^*}\left(z\right)\right) \\
\geq 1-\frac{8 R^2 w^{2}}{\left(\mathbb{E}\left[\left|\beta_{1}\right|\right] \right)^{2} \mathcal{B}\left(\frac{1}{2}, \frac{d-1}{2}\right)^{2}\left\|\nabla M_{c^*}\left(z\right)\right\|_{2}^{2}} \\
=1-\frac{2M_{c^*}\left(z\right)^2 R^2 (d-1)^2}{R^2 \delta^2 R^2 \left\| \nabla M_{c^*}\left(z\right) \right\|_{2}^2} \\
=1-\frac{2M_{c^*}\left(z\right)^2  (d-1)^2}{\delta^2 R^2 \left\| \nabla M_{c^*}\left(z\right) \right\|_{2}^2}
\end{split}
\end{align}
In addition, we have that
\begin{equation}
\mathbb{E} \left[ \widehat{\nabla M_{c^*}}\left(z, R\right) \right]=\mathbb{E}\left[\phi_{c^*}\left(z+\delta Ru\right) u\right],
\end{equation}
so plugging it to Ineq. \ref{eq:cosbound}, we have the final bound to be

\begin{align}
\begin{split}\label{eq:cosfin}
\cos \angle\left(\mathbb{E}\left[\widehat{\nabla M_{c^*}}\left(z, R\right)\right], \nabla M_{c^*}\left(z\right)\right) \\
\geq 1-\frac{9M_{c^*}\left(z\right)^2  (d-1)^2}{2\delta^2 R^4 \left\| \nabla M_{c^*}\left(z\right) \right\|_{2}^2}.
\end{split}
\end{align}
Therefore, we can observe that with increasing radius $R$, we will achieve
\begin{align}
\lim _{R \rightarrow \infty} \cos \angle\left(\mathbb{E}\left[\widehat{M_{c^*}}\left(z, R \right)\right], \nabla M_{c^{\star}}\left(z\right)\right)=1,
\end{align}
\end{proof}

\subsection{Nonlinear Case}
\label{sec:nonlinearcase}
\begin{theorem}
\label{thm:nonlinear}
Let $z$ be an arbitrary point within the target class, i.e. $M_{c^*}(z)>0$. Assume $f$ be a nonlinear classification model with a Lipschitz continuous gradient in a neighborhood of $z$. Then, the cosine of the angle between $\mathbb{E}[\widehat{M_{c^*}}(z, R)]$ and $\nabla M_{c^*}\left(z\right)$ is bounded by 
\begin{multline}
\cos \angle\left(\mathbb{E}\left[\widehat{\nabla M_{c^*}}\left(z, R\right)\right], \nabla M_{c^*}\left(z\right)\right) \\
\geq 1- \frac{4M_{c^*}\left(z\right)^2 + L^2\delta^4R^4 + 4M_{c^*}\left(z\right)^2 L \delta^2 R^2  }{2\delta^2  R^2 (d-1)^{-2}\left\|\nabla M_{c^*}\left(z\right)\right\|_{2}^{2}}.
\end{multline}

Therefore, with increasing radius $R$ up to the point of inflection $K$,
\begin{equation}
\lim _{R \rightarrow K} \cos \angle\left(\mathbb{E}\left[\widehat{M_{c^*}}\left(z, R \right)\right], \nabla M_{c^{\star}}\left(z\right)\right)=1,
\end{equation}
which tells that the alignment between the proposed gradient estimator and the true gradient $\nabla M_{c^*}\left(z\right)$ is increased with greater radius $R$ until reaching a certain point K. After passing that point of inflection, the alignment between them branches off.
\end{theorem}

\begin{proof}
Let $Ru$ be a random vector uniformly distributed on the sphere, where $R > 0$ is a radius of the sphere. By Taylor's theorem, for any $\delta \in(0,1)$, we have that
\begin{flalign} \begin{split} \label{eq:taylor2}
&M_{c^*}\left(z+\delta Ru\right) 
\\&=M_{c^*}\left(z\right) + \delta \nabla M_{c^*}\left(z\right)^{T} Ru + \frac{1}{2}\delta^2 \nabla^2 M_{c^*}\left(z\right)^{T} R^2u.
\end{split}
\end{flalign}
Since our function $f$ has a Lipschitz continuous gradient with Lipschitz constant $L >0$, the following inequality holds
\begin{align}
\begin{split}\label{eq:lipschitz2}
    \left\| \nabla^2 M_{c^*}\left(z\right)u \right\|_{2} & \leq L \\
    & \iff\\
    \left\| \frac{1}{2}\delta^2\nabla^2 M_{c^*}\left(z\right) R^2u \right\|_{2}  & \leq \frac{1}{2}L\delta^2R^2
\end{split}
\end{align}
Recall that 
$$M_{c^*}(z) > 0$$
and let $w:=\frac{M_{c^*}(z)}{\delta R} + \frac{1}{2}L\delta R.$ \\
For the case $\nabla M_{c^*}\left(z\right)^{T} u>w$, using Taylor series expansion and the fact that

\begin{align} 
\begin{split}\label{eq:gradgtw2}
&    \delta \nabla M_{c^*}\left(z\right)^{T} Ru + \frac{1}{2}\delta^2 \nabla^2 M_{c^*}\left(z\right)^{T} R^2u \\
&    \geq \delta \nabla M_{c^*}\left(z\right)^{T} Ru - \frac{1}{2}L\delta^2 R^2 \\
 &   > \delta \nabla M_{c^*}\left(z\right)^{T} Ru - \frac{1}{2}L\delta^2 R^2 - M_{c^*}\left(z\right) \\
&    = \delta R \left( \nabla M_{c^*}\left(z\right)^{T} u - w \right) > 0,
\end{split}
\end{align}
we derive that $M_{c^*}\left(z+\delta Ru\right) > M_{c^*}(z)$. \\
Similarly, for the case $\nabla M_{c^*}\left(x_{t}\right)^{T} u <- w$, using Taylor expansion and the fact that
\begin{align}
\begin{split}\label{eq:gradltw2}
    \delta \nabla M_{c^*}\left(z\right)^{T} Ru + \frac{1}{2}\delta^2 \nabla^2 M_{c^*}\left(z\right)^{T} R^2u \\
    \leq \delta \nabla M_{c^*}\left(z\right)^{T} Ru + \frac{1}{2}L\delta^2 R^2 \\
    < \delta \nabla M_{c^*}\left(z\right)^{T} Ru + \frac{1}{2}L\delta^2 R^2 + M_{c^*}\left(z\right) \\
    = \delta R \left( \nabla M_{c^*}\left(z\right)^{T} u + w \right) < 0,
\end{split}
\end{align}
 we obtain that $M_{c^*}\left(z+\delta Ru\right) < M_{c^*}(z)$. \\
 Therefore, from these two cases, we arrive at
\begin{align} \label{eq:phi2}
\phi_{c^*}\left(z+\delta Ru\right)=\left\{\begin{array}{l}
0 \text { if } \nabla M_{c^*}\left(z\right)^{T} Ru>w \\
-1 \text { if } \nabla M_{c^*}\left(z\right)^{T} Ru<-w.
\end{array}\right.
\end{align}
We define  $v_{1}=\nabla S\left(z\right) /\left\|\nabla M_{c^*}\left(z\right)\right\|_{2}, v_{2}, \ldots, v_{d} $ by expanding the vector $\nabla M_{c^*}\left(z\right)$ to orthogonal bases in $\mathbb{R}^{d}$. Then, we can write a random vector $Ru=\sum_{i=1}^{d} \beta_{i} v_{i}$, where $\beta$ is uniformly distributed on the sphere of radius $R$. We construct an upper cap $E_{1}:=$ $\left\{\nabla M_{c^*}\left(z\right)^{T} Ru>w\right\}$, the annulus $E_{2}:=\left\{\left|\nabla M_{c^*}\left(z\right)^{T} Ru\right|<w. \right \}$, and the lower cap $E_{3}:=\left\{\nabla M_{c^*}\left(z\right)^{T} Ru<-w\right\} .$ Let $p:=\p\left(E_{2}\right)$ be the probability of event $E_{2}$, then
$\p\left(E_{1}\right)=\p\left(E_{3}\right)=(1-p) / 2 .$ 
For any $i \neq 1$ by symmetry:
\begin{align}
\e\left[\beta_{i} \mid E_{1}\right]=\e\left[\beta_{i} \mid E_{3}\right]=0.
\end{align}
Then, the expected value of the estimator becomes

\begin{align}
\begin{split}\label{eq:expect2}
\mathbb{E}\left[\phi_{c^*}\left(z+\delta Ru\right) u\right]= p \cdot\left(\mathbb{E}\left[\phi_{c^*}\left(z+\delta Ru\right) u \mid E_{2}\right]\right)\\
-\frac{1}{2} p \cdot \left(\mathbb{E}\left[\beta_{1} v_{1} \mid E_{1}\right] + \mathbb{E}\left[-\beta_{1} v_{1} \mid E_{3}\right]\right) \\
+\frac{1}{2}\mathbb{E}\left[\beta_{1} v_{1} \mid E_{1}\right]+\frac{1}{2}\mathbb{E}\left[-\beta_{1} v_{1} \mid E_{3}\right].
\end{split}
\end{align}
Now, we can bound the difference between $\mathbb{E}\left[\left|\beta_{1}\right| v_{1}\right]=\frac{\mathbb{E}\left[\left|\beta_{1}\right|\right]}{\left\|\nabla M_{c^*}\left(z\right)\right\|_{2}} \nabla M_{c^*}\left(z\right)$ and $\mathbb{E}\left[\phi_{c^*}\left(z+\delta Ru\right) u\right]$ using  Eq. \ref{eq:expect2}:

\begin{align}
\begin{split}\label{eq:expbound2}
 &\left\|  \mathbb{E}  \left[\phi_{c^*}\left(z  +\delta Ru\right) u\right]  -  \mathbb{E}\left[\left|\beta_{1}\right| v_{1}\right]\right\|_{2} \\
& \leq \left\| p \cdot\left(\mathbb{E}\left[\phi_{c^*}\left(z+\delta Ru\right) u \mid E_{2}\right]\right) \right.\\
& - \left. \frac{1}{2} p \cdot \left(\mathbb{E}\left[\beta_{1} v_{1} \mid E_{1}\right] + \mathbb{E}\left[-\beta_{1} v_{1} \mid E_{3}\right]\right) \right\|_{2} \\
& \leq p \cdot (R + \frac{1}{2}R + \frac{1}{2}R) = 2Rp.
\end{split}
\end{align}
In Inequality \ref{eq:expbound2}, we first substitute the LHS with  $\mathbb{E}\left[\left|\beta_{1}\right| v_{1}\right]=\frac{\mathbb{E}\left[\left|\beta_{1}\right|\right]}{\left\|\nabla M_{c^*}\left(z\right)\right\|_{2}} \nabla M_{c^*}\left(z\right)$, then square both sides, and lastly divide by  $\mathbb{E}\left[\left|\beta_{1}\right|\right]^2$ to derive the following

\begin{align}
\left\|\frac{\mathbb{E}\left[\phi_{c^*}\left(z+\delta Ru\right) u\right]}{\mathbb{E}\left[\left|\beta_{1}\right|\right]}-\frac{\nabla M_{c^*}\left(z\right)}{\left\|\nabla M_{c^*}\left(z\right)\right\|_{2}}\right\|_{2}^2 \leq \left(\frac{2Rp}{\mathbb{E}\left[\left|\beta_1 \right| \right]}\right)^2.
\end{align}

Using the property of the angle between two vectors, we obtain the cosine inequality:
\begin{align}
\begin{split}\label{eq:cos2}
\cos \angle\left(\mathbb{E}\left[\phi_{c^*}\left(z+\delta Ru\right) u\right], \nabla M_{c^*}\left(z\right)\right) \\
\geq 1-\frac{1}{2}\left(\frac{2 Rp}{\mathbb{E}\left[\left|\beta_{1}\right|\right]}\right)^{2}.
\end{split}
\end{align} 
The probability $p$ can be bounded using the fact that $\left\langle\frac{\nabla M_{c^*}\left(z\right)}{\left\|\nabla M_{c^*}\left(z\right)\right\|_{2}}, u\right\rangle^{2}$ is a Beta distribution $\mathcal{B}\left(\frac{1}{2}, \frac{d-1}{2}\right)$ :

\begin{align}
\begin{split}\label{eq:pbound2}
p &=\mathbb{P}\left(\left\langle\frac{\nabla M_{c^*}\left(z\right)}{\left\|\nabla M_{c^*}\left(z\right)\right\|_{2}}, u\right\rangle^{2} \leq \frac{w^{2}}{\left\|\nabla M_{c^*}\left(z\right)\right\|_{2}^{2}}\right) \\
& \leq \frac{2 w}{\mathcal{B}\left(\frac{1}{2}, \frac{d-1}{2}\right)\left\|\nabla M_{c^*}\left(z\right)\right\|_{2}}.
\end{split}
\end{align}
Substituting the Bound \ref{eq:pbound2} into equality \ref{eq:cos2}, we receive a bound for the cosine

\begin{align}
\begin{split}\label{eq:cosbound2}
&\cos \angle\left(\mathbb{E}\left[\phi_{c^*}\left(z+\delta Ru\right) u\right], \nabla M_{c^*}\left(z\right)\right) \\
&\geq 1-\frac{8 R^2 w^{2}}{\left(\mathbb{E}\left[\left|\beta_{1}\right|\right] \right)^{2} \mathcal{B}\left(\frac{1}{2}, \frac{d-1}{2}\right)^{2}\left\|\nabla M_{c^*}\left(z\right)\right\|_{2}^{2}} \\
&\geq 1-\frac{8M_{c^*}\left(z\right)^2 + 2L^2\delta^4R^4 + 8M_{c^*}\left(z\right)^2 L \delta^2 R^2 }{\delta^2  \left(\mathbb{E}\left[\left|\beta_{1}\right|\right] \right)^{2} \mathcal{B}\left(\frac{1}{2}, \frac{d-1}{2}\right)^{2}\left\|\nabla M_{c^*}\left(z\right)\right\|_{2}^{2}} \\
&\geq 1-\frac{4M_{c^*}\left(z\right)^2 + L^2\delta^4R^4 + 4M_{c^*}\left(z\right)^2 L \delta^2 R^2  }{2\delta^2  R^2 (d-1)^{-2}\left\|\nabla M_{c^*}\left(z\right)\right\|_{2}^{2}}.
\end{split}
\end{align}
In addition, we have that
\begin{equation}
\mathbb{E} \left[ \widehat{\nabla M_{c^*}}\left(z, R\right) \right]=\mathbb{E}\left[\phi_{c^*}\left(z+\delta Ru\right) u\right],
\end{equation}
so plugging it to inequality \ref{eq:cosbound2}, we have the final bound to be

\begin{align}
\begin{split}\label{eq:cosfin2}
&\cos \angle\left(\mathbb{E}\left[\widehat{\nabla M_{c^*}}\left(z, R\right)\right], \nabla M_{c^*}\left(z\right)\right) \\
&\geq 1-\frac{4M_{c^*}\left(z\right)^2 + L^2\delta^4R^4 + 4M_{c^*}\left(z\right)^2 L \delta^2 R^2  }{2\delta^2  R^2 (d-1)^{-2}\left\|\nabla M_{c^*}\left(z\right)\right\|_{2}^{2}}.
\end{split}
\end{align}
Therefore, we can observe that with increasing radius $R$, we will achieve
\begin{align}
\lim _{R \rightarrow K} \cos \angle\left(\mathbb{E}\left[\widehat{M_{c^*}}\left(z, R \right)\right], \nabla M_{c^{\star}}\left(z\right)\right)=1,
\end{align}
\end{proof}

\section{Experiments}
\subsection{Attacking MNIST}
We extend our experiments to include tasks other than facial recognition. Particularly, we perform our attack on MNIST dataset. We use identical attack model to the main experiments, in which the attacker only gets access to the target model's decision.  We also assume that that the attacker has access to a prior knowledge (public dataset). In our case, these are the digits from ``5" to ``9". The attacker's goal is to infer information on the private dataset (i.e., the digits from ``0" to ``4"). This is not trivial since the public knowledge available to the attacker contains only 5 classes which makes it harder for the attacker to generalize to other classes. Nevertheless, \AlgName was able to successfully attack 4 out of the 5 private classes. 

Similar to the main experiments, we train two models. The target model is a network with 2 CNN layers followed by two linear layers and the evaluation classifier has 3 CNN layers followed by two linear layers. As shown in Fig.\ref{fig:mnist-1}, the initial image generated by GAN for class ``0" resembles class ``6". This is expected, since the attacker has only knowledge of classes from ``5" to ``9". The final output however, looks more similar to ``0". Similar behavior can be seen when attacking class ``4". For classes ``2" and ``3", the initial images were noises, yet the we were still able to successfully attack class ``3".
\begin{figure}[!htbp]
  \centering
  \includegraphics[width=0.65\linewidth]{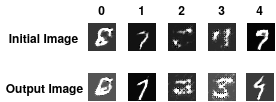}
  \caption{Performing \AlgName on MNIST dataset. The output is shows for each private label (digit).}
  \label{fig:mnist-1}
\end{figure}

\end{document}